\documentclass[english,final]{cvpr}
\usepackage[T1]{fontenc}
\usepackage[latin9]{inputenc}
\usepackage{units}
\usepackage{amsmath}
\usepackage{amsthm}
\usepackage{amssymb}
\usepackage[numbers]{natbib}

\makeatletter
\theoremstyle{plain}
\newtheorem{lem}{\protect\lemmaname}
\theoremstyle{plain}
\newtheorem{thm}{\protect\theoremname}
\ifx\proof\undefined\
  \newenvironment{proof}[1][\proofname]{\par
    \normalfont\topsep6\p@\@plus6\p@\relax
    \trivlist
    \itemindent\parindent
    \item[\hskip\labelsep
          \scshape
      #1]\ignorespaces
  }{%
    \endtrivlist\@endpefalse
  }
  \providecommand{\proofname}{Proof}
\fi
\theoremstyle{plain}
\newtheorem{cor}{\protect\corollaryname}
\theoremstyle{definition}
\newtheorem{defn}{\protect\definitionname}

\usepackage{times}
\usepackage{epsfig}
\usepackage{graphicx}
\usepackage{amsmath}
\usepackage{amssymb}

\usepackage[pagebackref=true,breaklinks=true,colorlinks,bookmarks=false]{hyperref}

\def\cvprPaperID{****} % *** Enter the CVPR Paper ID here
\def\confYear{CVPR 2021}
\makeatother

\usepackage{babel}
\providecommand{\corollaryname}{Corollary}
\providecommand{\definitionname}{Definition}
\providecommand{\lemmaname}{Lemma}
\providecommand{\theoremname}{Theorem}

\begin{document}
\global\long\def\R{\mathbb{R}}%

\global\long\def\C{\mathbb{C}}%

\global\long\def\N{\mathbb{N}}%

\global\long\def\e{{\mathbf{e}}}%

\global\long\def\et#1{{\e(#1)}}%

\global\long\def\ef{{\mathbf{\et{\cdot}}}}%

\global\long\def\x{{\mathbf{x}}}%

\global\long\def\xt#1{{\x(#1)}}%

\global\long\def\xf{{\mathbf{\xt{\cdot}}}}%

\global\long\def\d{{\mathbf{d}}}%

\global\long\def\w{{\mathbf{w}}}%

\global\long\def\b{{\mathbf{b}}}%

\global\long\def\u{{\mathbf{u}}}%

\global\long\def\y{{\mathbf{y}}}%

\global\long\def\n{{\mathbf{n}}}%

\global\long\def\k{{\mathbf{k}}}%

\global\long\def\yt#1{{\y(#1)}}%

\global\long\def\yf{{\mathbf{\yt{\cdot}}}}%

\global\long\def\z{{\mathbf{z}}}%

\global\long\def\v{{\mathbf{v}}}%

\global\long\def\h{{\mathbf{h}}}%

\global\long\def\s{{\mathbf{s}}}%

\global\long\def\c{{\mathbf{c}}}%

\global\long\def\p{{\mathbf{p}}}%

\global\long\def\f{{\mathbf{f}}}%

\global\long\def\rb{{\mathbf{r}}}%

\global\long\def\rt#1{{\rb(#1)}}%

\global\long\def\rf{{\mathbf{\rt{\cdot}}}}%

\global\long\def\mat#1{{\ensuremath{\bm{\mathrm{#1}}}}}%

\global\long\def\vec#1{{\ensuremath{\bm{\mathrm{#1}}}}}%

\global\long\def\matN{\ensuremath{{\bm{\mathrm{N}}}}}%

\global\long\def\matX{\ensuremath{{\bm{\mathrm{X}}}}}%

\global\long\def\X{\ensuremath{{\bm{\mathrm{X}}}}}%

\global\long\def\matK{\ensuremath{{\bm{\mathrm{K}}}}}%

\global\long\def\K{\ensuremath{{\bm{\mathrm{K}}}}}%

\global\long\def\matA{\ensuremath{{\bm{\mathrm{A}}}}}%

\global\long\def\A{\ensuremath{{\bm{\mathrm{A}}}}}%

\global\long\def\matB{\ensuremath{{\bm{\mathrm{B}}}}}%

\global\long\def\B{\ensuremath{{\bm{\mathrm{B}}}}}%

\global\long\def\matC{\ensuremath{{\bm{\mathrm{C}}}}}%

\global\long\def\C{\ensuremath{{\bm{\mathrm{C}}}}}%

\global\long\def\matD{\ensuremath{{\bm{\mathrm{D}}}}}%

\global\long\def\D{\ensuremath{{\bm{\mathrm{D}}}}}%

\global\long\def\matE{\ensuremath{{\bm{\mathrm{E}}}}}%

\global\long\def\E{\ensuremath{{\bm{\mathrm{E}}}}}%

\global\long\def\matF{\ensuremath{{\bm{\mathrm{F}}}}}%

\global\long\def\F{\ensuremath{{\bm{\mathrm{F}}}}}%

\global\long\def\matP{\ensuremath{{\bm{\mathrm{P}}}}}%

\global\long\def\P{\ensuremath{{\bm{\mathrm{P}}}}}%

\global\long\def\matU{\ensuremath{{\bm{\mathrm{U}}}}}%

\global\long\def\matV{\ensuremath{{\bm{\mathrm{V}}}}}%

\global\long\def\V{\ensuremath{{\bm{\mathrm{V}}}}}%

\global\long\def\matW{\ensuremath{{\bm{\mathrm{W}}}}}%

\global\long\def\matM{\ensuremath{{\bm{\mathrm{M}}}}}%

\global\long\def\M{\ensuremath{{\bm{\mathrm{M}}}}}%

\global\long\def\calC{{\cal C}}%

\global\long\def\calF{{\cal F}}%

\global\long\def\calH{{\cal H}}%

\global\long\def\calL{{\cal L}}%

\global\long\def\calG{{\cal G}}%

\global\long\def\calY{{\cal Y}}%

\global\long\def\calP{{\cal P}}%

\global\long\def\calX{{\cal X}}%

\global\long\def\calS{{\cal S}}%

\global\long\def\calT{{\cal T}}%

\global\long\def\calU{{\cal U}}%

\global\long\def\calV{{\cal V}}%

\global\long\def\calW{{\cal W}}%

\global\long\def\calZ{{\cal Z}}%

\global\long\def\Normal{{\cal \mathcal{N}}}%

\global\long\def\matQ{{\mat Q}}%

\global\long\def\Q{{\mat Q}}%

\global\long\def\matR{\mat R}%

\global\long\def\matS{\mat S}%

\global\long\def\matY{\mat Y}%

\global\long\def\matI{\mat I}%

\global\long\def\I{\mat I}%

\global\long\def\matJ{\mat J}%

\global\long\def\matZ{\mat Z}%

\global\long\def\Z{\mat Z}%

\global\long\def\matW{{\mat W}}%

\global\long\def\W{{\mat W}}%

\global\long\def\matL{\mat L}%

\global\long\def\S#1{{\mathbb{S}_{N}[#1]}}%

\global\long\def\IS#1{{\mathbb{S}_{N}^{-1!}[#1]}}%

\global\long\def\PN{\mathbb{P}_{N}}%

\global\long\def\abs#1{\vert#1\vert}%

\global\long\def\Norm#1{\left\Vert #1\right\Vert }%

\global\long\def\NormH#1{\left\Vert #1\right\Vert _{\calH}}%

\global\long\def\TNormS#1{\|#1\|_{2}^{2}}%

\global\long\def\ITNormS#1{\|#1\|_{2}^{-2}}%

\global\long\def\ONorm#1{\|#1\Vert_{1}}%

\global\long\def\TNorm#1{\|#1\|_{2}}%

\global\long\def\InfNorm#1{\|#1\|_{\infty}}%

\global\long\def\FNorm#1{\|#1\|_{F}}%

\global\long\def\FNormS#1{\|#1\|_{F}^{2}}%

\global\long\def\UNorm#1{\|#1\|_{\matU}}%

\global\long\def\UNormS#1{\|#1\|_{\matU}^{2}}%

\global\long\def\UINormS#1{\|#1\|_{\matU^{-1}}^{2}}%

\global\long\def\ANorm#1{\|#1\|_{\matA}}%

\global\long\def\BNorm#1{\|#1\|_{\mat B}}%

\global\long\def\ANormS#1{\|#1\|_{\matA}^{2}}%

\global\long\def\AINormS#1{\|#1\|_{\matA^{-1}}^{2}}%

\global\long\def\T{\textsc{T}}%

\global\long\def\conj{\textsc{*}}%

\global\long\def\pinv{\textsc{+}}%

\global\long\def\Expect#1{\operatorname{E}\left[#1\right]}%

\global\long\def\ExpectC#1#2{\operatorname\{E\}_{#1}\left[#2\right]}%

\global\long\def\Prob#1{\operatorname{P}\left[#1\right]}%

\global\long\def\Var#1{{\mathbb{\mathrm{Var}}}\left[#1\right]}%

\global\long\def\VarC#1#2{{\mathbb{\mathrm{Var}}}_{#1}\left[#2\right]}%

\global\long\def\dotprod#1#2#3{(#1,#2)_{#3}}%

\global\long\def\dotprodN#1#2{(#1,#2)_{{\cal N}}}%

\global\long\def\dotprodH#1#2{\langle#1,#2\rangle_{{\cal {\cal H}}}}%

\global\long\def\dotprodsqr#1#2#3{(#1,#2)_{#3}^{2}}%

\global\long\def\Trace#1{{\bf Tr}\left(#1\right)}%

\global\long\def\nnz#1{{\bf nnz}\left(#1\right)}%

\global\long\def\MSE#1{{\bf MSE}\left(#1\right)}%

\global\long\def\WMSE#1{{\bf WMSE}\left(#1\right)}%

\global\long\def\EWMSE#1{{\bf EWMSE}\left(#1\right)}%

\global\long\def\excpectedloss#1{{\cal L}\left(#1\right)}%

\global\long\def\nicehalf{\nicefrac{1}{2}}%

\global\long\def\argmin{\operatornamewithlimits{argmin}}%

\global\long\def\argmax{\operatornamewithlimits{argmax}}%

\global\long\def\norm#1{\Vert#1\Vert}%

\global\long\def\pred{\operatorname{pred}}%

\global\long\def\sign{\operatorname{sign}}%

\global\long\def\diag{\operatorname{diag}}%

\global\long\def\VOPT{\operatorname\{VOPT\}}%

\global\long\def\dist{\operatorname{dist}}%

\global\long\def\diag{\operatorname{diag}}%

\global\long\def\sp{\operatorname{span}}%

\global\long\def\onehot{\operatorname{onehot}}%

\global\long\def\softmax{\operatorname{softmax}}%

\newcommand*\diff{\mathop{}\!\mathrm{d}} 

\global\long\def\dd{\diff}%

\global\long\def\whatlambda{\w_{\lambda}}%

\global\long\def\Plambda{\mat P_{\lambda}}%

\global\long\def\Pperplambda{\left(\mat I-\Plambda\right)}%

\global\long\def\Mlambda{\matM_{\lambda}}%

\global\long\def\Mlambdafull{\matM+\lambda\matI}%

\global\long\def\Mlambdafullinv{\left(\matM+\lambda\matI\right)^{-1}}%

\global\long\def\Mdaggerlambda{{\mathbf{\mat M_{\lambda}^{+}}}}%

\global\long\def\Xdaggerlambda{{\mathbf{\mat X_{\lambda}^{+}}}}%

\global\long\def\XT{{\mathbf{X}^{\T}}}%

\global\long\def\XXT{{\matX\mat X^{\T}}}%

\global\long\def\XTX{{\matX^{\T}\mat X}}%

\global\long\def\VT{{\mathbf{V}^{\T}}}%

\global\long\def\VVT{{\matV\mat V^{\T}}}%

\global\long\def\VTV{{\matV^{\T}\mat V}}%

\global\long\def\varphibar#1#2{{\bar{\varphi}_{#1,#2}}}%

\global\long\def\varphilambda{{\varphi_{\lambda}}}%

\global\long\def\lunhi{{l_{\mathrm{unh}}^{i}}}%

\global\long\def\lunh{{l_{\mathrm{unh}}}}%

\global\long\def\pardiff#1#2{{\frac{\partial#1}{\partial#2}}}%

\global\long\def\nicepardiff#1#2{{\nicefrac{\partial#1}{\partial#2}}}%

%%%%%%%%% TITLE
\title{An Exploration into why Output Regularization Mitigates Label Noise }
\author{Neta Shoham\\
Edgify\\
{\tt\small neta.shoham@edgify.ai}\\
% Additional authors and addresses can be added with ``\and'', \
\and
Tomer Avidor\\
Edgify\\
{\tt\small tomer.avidor@edgify.ai}\\
\and
Nadav Israel\\
Edgify\\
{\tt\small nadav.israel@edgify.ai}\\
}
\maketitle
\thispagestyle{empty} 
\begin{abstract}
Label noise presents a real challenge for supervised learning algorithms.
Consequently, mitigating label noise has attracted immense research
in recent years. Noise robust losses is one of the more promising
approaches for dealing with label noise, as these methods only require
changing the loss function and do not require changing the design
of the classifier itself, which can be expensive in terms of development
time. In this work we focus on losses that use output regularization
(such as label smoothing and entropy). Although these losses perform
well in practice, their ability to mitigate label noise lack mathematical
rigor. In this work we aim at closing this gap by showing that losses,
which incorporate an output regularization term, become symmetric
as the regularization coefficient goes to infinity. We argue that
the regularization coefficient can be seen as a hyper-parameter controlling
the symmetricity, and thus, the noise robustness of the loss function.
\end{abstract}

\section{Introduction}

Tagging of unlabeled data is an expensive and time consuming effort.
The desire to cut costs and speedup the tagging process roften results
in data that suffers from label noise. The ability to learn under
label noise is thus one of the most common machine learning problems,
when it comes to real life scenarios. As such, mitigation of label
noise has attracted a vast amount of machine learning research, \citep{frenay2013classification}
with a recent surge in the field of deep learning \citep{algan2019image,song2020learning}.
In this work we only consider uniform (sometimes called symmetric)
label noise.

\paragraph{The Robust Loss Approach: }

A branch of methods promote learning with robust loss functions \citep{ghosh2017robust,van2015learning,wang2019symmetric,zhang2018generalized,manwani2013noise,ma2020normalized,charoenphakdee2019symmetric,ishida2019complementary}.
A loss function is said to be robust if the same test accuracy is
achieved when a model is trained with noisy or clean data. The robust
loss approach is appealing since it only requires replacing of one
loss with another. Recent works show that this ``simplicity'' does
not come on the expense of accuracy \citep{wang2019symmetric,ma2020normalized}.
Moreover, robust loss functions, which are based on symmetric loss
functions have the advantage of being based on strong theoretical
foundations \citep{ghosh2017robust,van2015learning,chou2020unbiased}. 

\paragraph*{Output Regularization and Confidence Penalty:}

Adding an output regularization is another approach that can be considered
for label noise mitigation. Output regularization, when applied to
the output of a softmax layer, can also be thought of as a confidence
penalty. Intuitively, that prevents the training process from overfitting
on falsely labeled data. For example, label smoothing, that can be
seen as form of confidence penalty \citep{pereyra2017regularizing},
was shown experimentally to mitigate label noise \citep{lukasik2020does}.
However, a full theoretical explanation of the merits of output regularization
in the context of label noise is still missing.

The main contribution of this work is a novel theoretical explanation
for the success of output regularization in the mitigation of label
noise. This explanation is based an a novel discovery that in many
cases the loss function becomes asymptotically symmetric, and thus
robust, when the coefficient of the output regularization goes to
infinity. 

\section{Notations and Preliminaries}

Let $X\in\calX$, $Y\in\left[C\right]:=\{1,\dots,C\}$ be jointly
distributed random variables and let $\bar{Y}$ be a noisy version
of $Y$ such that for some $\rho<\frac{C-1}{C}$. 
\[
P\left(\bar{Y}=i\mid Y,X\right)=\begin{cases}
\frac{\rho}{C-1} & i\ne Y\\
1-\rho & i=Y
\end{cases}.
\]

Let $\calF\subseteq\left[C\right]^{\calX}$ be a family of measurable
functions and define a set of random variable ${\cal Z=}\left\{ f\left(X\right)\mid f\in{\cal F}\right\} $.
For a loss function $l:\R^{C}\times\left[C\right]\to\R$ and $Z\in\calZ$
we define $L_{l}\left(Z\right)=\Expect{l\left(Z,Y\right)}$, $\bar{L}_{l}\left(Z\right)=\Expect{l\left(Z,\bar{Y}\right)}$. 

We say that a loss function $l$ is robust if 
\begin{multline*}
\bar{Z}^{\star}\in\argmin_{Z\in\calZ}\bar{L}_{l}\left(Z\right)\implies\exists Z^{\star}\in\argmin_{Z\in\calZ}L_{l}\left(Z\right),\\
\pred\left(\bar{Z}^{\star}\right)\stackrel{a.s.}{=}\pred\left(Z^{\star}\right),
\end{multline*}
where $\pred\left(z\right)=\argmax_{i\in\left[C\right]}\left(z_{i}\right)$.
Note that 
\[
\argmin_{Z\in\calZ}\bar{L}_{l}\left(Z\right)\subseteq\argmin_{Z\in\calZ}L_{l}\left(Z\right)
\]
is sufficient for $l$ to be robust.

A regularizer $g$ is any function from $\R^{C}$ to $\R$ that has
a single minima at $0\in\R^{C}$. \footnote{We note that confidence penalty regularizers, like label smoothing
and entropy are not compiled with this definition. At the end of this
article we have a special treatment for these regularizers.}For a regularizer $g$ and $Z\in{\cal Z}$ we use the notation: $G_{g}\left(Z\right)=\Expect{g\left(Z\right)}$.

When $\calF=\left\{ f_{\theta}\right\} _{\theta\in\Theta}$ we use
the following notations: $\calL_{l}\left(\theta\right)=L_{l}\left(Z\left(\theta\right)\right)$
and $\calG_{g}\left(\theta\right)=G_{g}\left(Z\left(\theta\right)\right)$,
where $Z\left(\theta\right)=f_{\theta}\left(X\right).$

Finally, for a random variable $Z$ we define the $L^{2}$ norm $\left\Vert Z\right\Vert _{L^{2}}=\sqrt{\Expect{\left\Vert Z\right\Vert ^{2}}}.$

\section{Related Work}

\paragraph{Symmetric Loss Functions:}

A sufficient property for a loss function $l:\R^{C}\times[C]\to\R$
to be robust is that $\sum_{y\in\left[C\right]}l\left(z,y\right)$
is constant (does not depend on $z$) \citep{ghosh2017robust}. A
loss function that has this property is called a symmetric loss. For
a softmax output $p$, \citet{ghosh2017robust} proposed the \emph{Mean
Absolute Error} (MAE):
\[
l\left(p,y\right)=\norm{p-\onehot\left(y\right)}_{1}=2\left(1-p_{y}\right)
\]
as a symmetric loss function.

\paragraph*{Asymptotically Symmetric Losses:}

Recent works showed an inefficiency of the MAE loss from a gradient
decent perspective \citep{zhang2018generalized,wang2019symmetric}
. To solve this issue \citet{zhang2018generalized} proposed the\emph{
Generalized Cross Entropy }(GCE) loss function:
\[
l_{\text{q}}\left(p,y\right)=\frac{1-p_{y}^{q}}{q},\ q\in(0,1]
\]
as a sort of compromise between MAE and \emph{Cross Entropy }(CE).
Note that when $q\to0$ GCE converges to\emph{ }CE, while it is proportional
to MAE when $q\to1$. Later, motivated by the idea of \emph{Reversed
Cross Entropy }(RCE), \citet{wang2019symmetric} suggested the\emph{
Symmetric Cross Entropy }(SCE). This loss can be practically written
as

\[
l_{\lambda}\left(p,y\right)=-\log\left(p_{y}\right)+\text{\ensuremath{\lambda\left(1-p_{y}\right)}}.
\]
Again, when $\lambda=0$ it is just equal to the standard CE, and
and when $\lambda=1$ it is proportional to MAE.

\paragraph*{Mitigation of Label Noise with Regularization:}

\citet{hu2019simple} gave proven generalization guarantees for deep
learning using two special regularization techniques. However, they
assumed an additive noise and a linear model, where the connection
to deep learning is achieved only through the notion of \emph{Neural
Tangent Kernel} (NTK) \citep{jacot2018neural,lee2019wide,arora2019exact}
and requires a very wide network. \citet{lukasik2020does} showed
empirically, that confidence penalty (in the form of label smoothing)
can mitigate label noise. They further showed that label smoothing
can be translated to parameter regularization. The question of how
that leads to a label noise robustness was left open, however.

\paragraph{Asymptotical Robustness of Losses with an $l_{2}$ Regularized Parameter:}

Asymptotical symmetricity of convex binary losses, combined with $l_{2}$
regularization on the parameter was discussed in \citet{van2015learning}.
However, it seems that this concept has not yet been developed further
to multi-class classification, and to output regularization.

\section{The Multi-Category Unhinged loss }

Let us define a new symmetric loss function, the\emph{ Multi-Category
Unhinged} (MUH) loss:
\[
l^{\text{MUH}}\left(z,y\right)=\frac{1}{C}\sum_{i\in\left[C\right]}z_{i}-z_{y}=-\left\langle z,\onehot^{\star}\left(y\right)\right\rangle 
\]
where, 
\[
\onehot^{\star}\left(y\right)=\onehot\left(y\right)-\frac{1}{C}.
\]
This loss can be seen as a multi-categorial extension of the binary\emph{
unhinged loss,
\[
l\left(z,y\right)=1-yz,\ \ y\in\left\{ -1,1\right\} 
\]
}presented in \citet{van2015learning}, that in turn was inspired
by the\emph{ hinge} loss (made famous by its use in SVM)\emph{:}

\[
l\left(z,y\right)=\max\left(1-yz,0\right)\ \ y\in\left\{ -1,1\right\} .
\]

Our intent in proposing this new loss, is to later use it to derive
conditions under which regularized losses are asymptotically robust
to label noise.

Since the MUH loss is unbounded from bellow, it makes sense to add
a regularization term to it. Intuitively, adding a regularization
term to a symmetric loss should not harm robustness by much, as the
regularization term does not depend on the labels at all. This is
made clear by the following lemma which is true for any symmetric
loss:

\begin{lem}
\label{lem:allmost-robust}Let $l$ be a symmetric loss function and
let $g$ be a regularizer. Let 
\begin{equation}
\lambda=\frac{1}{a},\text{\ where }a=1-\frac{\rho C}{C-1}>0\label{eq:lambda-a}
\end{equation}
Then
\[
\argmin_{Z\in\calZ}\bar{L}_{l}\left(Z\right)+G_{g}\left(Z\right)\subseteq\argmin_{Z\in\calZ}L_{l}\left(Z\right)+\lambda G_{g}\left(Z\right)
\]
\end{lem}
The following theorem shows that in the special case where $l=l^{\text{MUH}}$
and $g$ is quadratic with a positive definite Hessian, the regularized
loss $l+g$ is robust under a wide family of models which includes
neural networks:
\begin{thm}
\label{thm:unhinged-square}Assume that for all $t>0$
\[
Z\in\calZ\iff tZ\in\calZ.
\]
If $g\left(z\right)=z^{T}Az$ for some positive definite matrix $A\in\R^{C\times C}$,
then $l^{\text{MUH}}+g$ is robust.
\begin{proof}
Assume that $Z^{\star}$ minimizes:
\begin{multline*}
\bar{L}_{l}\left(Z\right)+\bar{G}_{g}\left(Z\right)=\Expect{-Z^{T}\onehot^{\star}\left(\bar{Y}\right)+Z^{T}AZ}\\
\text{s.t.}\ Z\in\calZ
\end{multline*}
Define $\lambda$ as in eq \ref{eq:lambda-a}. then by theorem \ref{lem:allmost-robust}
$Z^{\star}$ also minimizes:
\begin{multline*}
L_{l}\left(Z\right)+\lambda G_{g}\left(Z\right)=\Expect{-Z^{T}\onehot^{\star}\left(Y\right)+\lambda Z^{T}AZ}\\
=\Expect{-\frac{1}{2\lambda}\Norm{\sqrt{2}\lambda A^{\nicehalf}Z-A^{-\nicehalf}\onehot^{\star}\left(Y\right)}^{2}}+\text{const}\\
\text{s.t.\ }Z\in\calZ
\end{multline*}
From this and from our assumption that implies
\[
\calZ=\left\{ \lambda Z\mid Z\in\calZ\right\} .
\]
we have that $\tilde{Z}=\lambda Z^{\star}$ minimizes:
\begin{multline*}
\Expect{-\frac{1}{2\lambda}\Norm{\sqrt{2}A^{\nicehalf}Z-A^{-\nicehalf}\onehot^{\star}\left(Y\right)}^{2}}\\
=\frac{1}{\lambda}\Expect{-Z^{T}\onehot^{\star}\left(Y\right)+Z^{T}AZ}+\text{const}\\
\text{s.t. }Z\in\calZ
\end{multline*}
Thus, $\tilde{Z}$ also minimizes $L_{l}\left(Z\right)+G_{g}\left(Z\right)$,
and since $\pred\left(Z\right)=\pred\left(\tilde{Z}\right)$, we have
what we need.
\end{proof}
\end{thm}
Using the last theorem we can now generalize a result by \citet{manwani2013noise}
which showed that the binary square loss is robust for the linear
classifiers family. 
\begin{cor}
If $\calZ$ is such that $Z\in\calZ\iff tZ\in\calZ$ for all $t>0$,
then the square loss $l\left(z,y\right)=\Norm{z-y}^{2}$ is robust
to uniform label noise.
\end{cor}
\begin{proof}
Just substitute $A=\nicehalf I$ and use the fact that 
\[
\frac{1}{2}\norm{z-\onehot^{*}\left(y\right)}^{2}=l^{\text{MUH}}\left(z,y\right)+\frac{1}{2}\norm z^{2}+\text{const}.
\]
\end{proof}
Note that we generalize the result of \citet{manwani2013noise} in
both: the number of classes and the variety of models. In addition
we note that our result gives a novel justification for using $\onehot^{\star}\left(y\right)$
instead of $\onehot\left(y\right)$ when using the square loss for
multi-category classification.

\section{\label{sec:Asymptotical-Robustness}Asymptotical Robustness}

In the previous section we established the robustness of the MUH loss
with quadratic output regularization. In this section we use this
fact to show that an important class of regularized loss functions
become asymptotically robust when the regularization coefficient goes
to infinity. This class is comprised of loss functions of the form
$l_{\lambda}\left(z,y\right)=l\left(z,y\right)+\lambda g\left(z\right)$,
where $\nabla_{z}l\left(0,y\right)=\nabla_{z}l^{\text{MUH}}\left(0,y\right)$
and $g$ is twice differentiable. 

If $Z_{\lambda}^{*}$ minimizes 
\[
L_{l_{\lambda}}\left(Z\right)=L_{l}\left(Z\right)+\lambda G_{g}\left(Z\right)
\]
Then it also minimizes 
\[
L_{l}\left(Z\right)\ \ \text{s.t.}\ G_{g}\left(Z\right)\le a_{\lambda}:=G\left(Z_{\lambda}^{*}\right)
\]
Thus, intuitively we expect that when $\lambda\to\infty$ (and $a_{\lambda}\to0$)
$Z$ will converge towards zero, where by Taylor approximation 
\begin{gather}
l\left(z,y\right)\approx l^{\text{lin}}\text{\ensuremath{\left(z,y\right)}}:=\nabla_{z}l\left(0,y\right)z=l^{\text{MUH}}\left(z,y\right)\label{eq:g^sq}\\
\text{ and }g\left(z\right)\approx g^{\text{sq}}\left(z\right):=z^{T}\nabla^{2}g\left(0\right)z.\nonumber 
\end{gather}
Using Theorem \ref{thm:unhinged-square} we can now conclude that
$l_{\lambda}$ is robust when $\lambda\to\infty$.

We will now give a formal definition of asymptotical robustness. It
will be convenient to work with losses of the form $l_{\alpha}=\alpha l+g$
and let $\alpha\to0$. The asymptotical results are equivalent to
those achieved with $l_{\lambda}=l+\lambda g$ and $\lambda\to\infty$.
We will also assume from now that $\calF=\left\{ f_{\theta}\mid\theta\in\R^{m}\right\} $.

When $\calL_{l_{\alpha}}$$\left(\theta\right)$ has a single minimizer
(for example, when $\calF$ is the family of linear classifiers and
$l_{\alpha}$ is strongly convex) the following definition is useful:
\begin{defn}
Let $l_{\alpha}$ be a loss function such that ${\cal L}_{l_{\alpha}}\left(\theta\right)$
has a unique minimizer $\theta_{\alpha}$. We say that $l_{\alpha}$
is asymptotically robust if there exists a robust loss function $\hat{l}_{\alpha}$
such that ${\cal L}_{\hat{l}_{\alpha}}\left(\theta\right)$ has a
unique minimizer $\hat{\theta}_{\alpha}$ and
\begin{equation}
\Norm{\frac{Z\left(\theta_{\alpha}\right)}{\Norm{Z\left(\theta_{\alpha}\right)}_{L^{2}}}-\frac{Z\left(\hat{\theta}_{\alpha}\right)}{\Norm{Z\left(\hat{\theta}_{\alpha}\right)}_{L^{2}}}}_{L^{2}}\xrightarrow[\alpha\to0]{}0,\label{eq:Z-to-Zhat-alpha}
\end{equation}
\end{defn}
In the general case, where $\calL_{l_{\alpha}}$ doesn't have a unique
minimizer we need to use a weaker definition:
\begin{defn}
\label{def:assymptotical-robustness}Let $l_{\alpha}$ be a loss function.
We say that $l_{\alpha}$ is asymptotically locally robust at $\theta_{0}$
if there exists a robust loss function $\hat{l}_{\beta}$ and there
exists $0<\beta_{n}\to0$, $\hat{\theta}_{n}\to\theta_{0}$ a local
minimizer of ${\cal L}_{\hat{l}_{\beta_{n}}}$, such that for any
$0<\alpha_{n}\to0$ and $\theta_{n}\to\theta_{0}$ a minimizer of
$\calL_{l_{\alpha_{n}}}$ it holds that:
\begin{equation}
\Norm{\frac{Z\left(\theta_{n}\right)}{\Norm{Z\left(\theta_{n}\right)}_{L^{2}}}-\frac{\hat{Z}\left(\hat{\theta}_{n}\right)}{\Norm{\hat{Z}\left(\hat{\theta}_{n}\right)}_{L^{2}}}}_{L^{2}}\xrightarrow[n\to\infty]{}0.\label{eq:Z-to-Zhat}
\end{equation}
\end{defn}

We are now ready to present our main theorem:
\begin{thm}
\label{thm:assymptotic-local-robustness}Let $l$ be a twice continuously
differentiable loss function such that $\nabla_{z}l\left(0,y\right)=\nabla_{z}l^{\text{MUH}}\left(0,y\right)$
and let $g$ be a twice continuously differentiable regularizer. Let
$\theta_{0}\in\R^{m}$ be such that $Z\left(\theta_{0}\right)\stackrel{\text{a.s.}}{=}0$
. Assume that $Z\left(\theta\right)$ is almost surely twice continuously
differentiable on the closure of an open neighborhood $\Omega$ of
$\theta_{0}$. Also assume that $\Expect{\nabla^{T}Z\left(\theta\right)\nabla Z\left(\theta\right)}\in\R^{m\times m}$
is positive definite on $\Omega$. If in addition $\nabla Z\left(\theta_{0}\right)\left[\nabla^{2}\calG\left(\theta_{0}\right)\right]^{-1}\nabla\calL_{l}\left(\theta_{0}\right)\stackrel{\text{a.s.}}{\neq}0$,
it holds that $l_{\alpha}=\alpha l+g$ is asymptotically locally robust
at $\theta_{0}$.
\end{thm}
Equipped with Theorem \ref{thm:assymptotic-local-robustness}. it
is now easy to prove a (global) asymptotical robustness when ${\cal F}$
is the family of linear classifiers and $\nabla_{z}l\left(0,y\right)=\nabla_{z}l^{\text{MUH}}\left(0,y\right)$
:
\begin{thm}
\label{thm:strong-robustness}Let $l$ be a twice continues differentiable
convex loss function such that $\nabla_{z}l\left(0,y\right)=\nabla_{z}l^{\text{MUH}}\left(0,y\right)$
and let $g$ be a twice continuously differentiable regularizer. Assume
that $\calF=\left\{ x\mapsto\left(\theta_{1}^{T}\phi\left(x\right),\dots,\theta_{C}^{T}\text{\ensuremath{\phi\left(x\right)}}\right)\mid\theta_{i}\in\R^{\frac{m}{C}}\right\} ^{T}$,
for some $\phi:\calX\to\R^{\frac{m}{C}}$ such that $\Expect{\phi\left(X\right)\phi\left(X\right)^{T}}$
is positive definite, then $l$ is asymptotically robust to uniform
label noise.
\end{thm}

\section{Softmax Cross Entropy}

CE is a very common loss function in modern machine learning. Naturally,
we would like to know if adding regularization to CE makes it asymptotically
robust. Usually CE is preceded by a softmax layer. We can thus consider
them together as one loss function: 

\[
l\left(z,y\right)=-z_{y}+\log\left(\sum_{i\in\left[C\right]}e^{z_{i}}\right).
\]
In this case we have that 
\[
\nabla_{z}l\left(0,y\right)=-\onehot^{*}\left(y\right)=\nabla_{z}l^{\text{MUH}}\left(0,y\right),
\]
which is exactly what we need for asymptotical robustness\emph{.}

If we only consider the pure CE (without softmax), $l\left(p,y\right)=-\log\left(p_{y}\right)$,
we need to take into account that its domain is the simplex $\left\{ p\succ0\in\R^{C}\mid\Norm p_{1}=1\right\} .$
Currently, our theory only supports losses which have the full $\R^{C}$
as a domain. In addition, in this case, common output regularizers
pushes the $p$ towards $\left(\nicefrac{1}{C},\dots,\nicefrac{1}{C}\right)^{T}$and
not towards 0. However, we can still check if the first order Taylor
approximation is symmetric near $p=\left(\nicefrac{1}{C},\dots,\nicefrac{1}{C}\right)^{T}$
which hints at the existence of possible asymptotical robustness,
and indeed:
\begin{eqnarray*}
\sum_{i\in\left[C\right]}l^{\text{lin}}(p,i) & = & \sum_{i\in\left[C\right]}\nabla_{z}l\left(\left(\nicefrac{1}{C},\dots,\nicefrac{1}{C}\right)^{T},i\right)p\\
 & = & \sum_{i\in\left[C\right]}-Cp_{i}=-C
\end{eqnarray*}

\paragraph{Handling Confidence Penalty Regularizers:}

When a regularizer has the form of $g\left(z\right)=h\left(\softmax\left(z\right)\right)$
it can be considered as a confidence penalty. Common confidence penalties
are, for example, entropy: $h\left(p\right)=\sum_{i\in\left[C\right]}p_{i}\log\left(p_{i}\right)$
and label-smoothing: $h\left(p\right)=\sum_{i\in\left[C\right]}\log\left(p_{i}\right)$
\citep{pereyra2017regularizing}. These output regularizers do not
have a unique minimum at $0$, but rather they are minimized on the
line 
\[
{\cal A}=\left\{ z\in\R^{C}\mid z_{1}=\dots=z_{C}\right\} .
\]
 Thus, even in the linear case, there is no unique solution $Z_{\alpha}\to0$
to the problem
\[
\min_{Z\in\calZ}L_{l+\alpha g}\left(Z\right).
\]
Even if $l$ is convex. 

Roughly speaking, a possible way to overcome this difficulty, is to
apply the whole theory of Section \ref{sec:Asymptotical-Robustness}
on an alternative output space $\left\{ z\in\R^{C}\mid z_{C}=0\right\} $
which is generated by subtracting the last coordinate of each output
from the rest of its coordinates. This operation does not change the
result of the $\softmax$ function and the new output space is isomorphic
to $\R^{C-1}$ which allows us to apply the theory.

\section{Conclusion}

In this work we proposed a new concept of asymptotical robustness
to uniform label noise. We showed that asymptotical robustness exists,
and suggested the idea that this is what stands behind the success
of output regularization methods to mitigate label noise. The cornerstone
of our theory is the robustness of the MUH loss with a quadratic output
regularizer. As a byproduct of this robustness we also proved that
the square loss, with a modified $\onehot$ function, is robust to
uniform label noise under a classifiers family that can be represented
by a neural network. At the end we showed that the softmax-CE loss
can be asymptotically robust if is equipped with a confidence penalty
regularizer. 

\bibliographystyle{plainnat}
\bibliography{bibtex}

\newpage\clearpage\pagebreak{}

\appendix

\section{Proofs}
\begin{lem}
\label{lem:ghosh}Let $l$ be a loss function. Assume that $\bar{Y}$
is corrupted with uniform label noise of level $\rho$ then for any
$Z\in\calZ$
\begin{gather*}
\bar{L}_{l}\left(Z\right)=\frac{\rho K}{C-1}+\left(1-\frac{\rho C}{C-1}\right)L_{l}\left(Z\right)
\end{gather*}
\end{lem}
\begin{proof}
This Lemma is proved in \citet{ghosh2017robust} as a part of their
main theorem
\end{proof}
\begin{lem}
\label{lem:convergance_in_cos-1}Let $h,g:\Theta\subseteq\R^{m}\to\R$
. Assume that $\theta_{0}$ is a minimizer of $g$ and that $g$ is
strictly convex and twice continuously differentiable at $\theta_{0}$.
Assume that $h$ is twice continuously differentiable at $\theta_{0}$
such that $\nabla h\left(\theta_{0}\right)\ne0$. Let $0\ne\alpha_{n}\to0,$
and assume that $\theta_{n}\in\argmin\alpha_{n}h+g$ such that $f$
and $g$ are differentiable at $\theta_{n}$, then $\frac{\theta_{n}-\theta_{0}}{\Norm{\theta_{n}-\theta_{0}}}\to\frac{v}{\Norm v}$
where $v:=-\left[\nabla^{2}g\left(\theta_{0}\right)\right]^{-1}\nabla h\left(\theta_{0}\right).$
\end{lem}
\begin{proof}
Let $\delta_{n}=\theta_{n}-\theta_{0}$. Since $h$ and $g$ are differentiable
at $\theta_{n}$:

\[
\alpha_{n}\nabla h\left(\theta_{n}\right)=-\nabla g\left(\theta_{n}\right)
\]
From the continues second differentiability of $h$ and $g$ at $\theta_{0}$
and form the fact that $\nabla g\left(\theta_{0}\right)=0$ we have:
\begin{gather*}
\alpha_{n}\left[\nabla h\left(\theta_{0}\right)+\nabla^{2}h\left(\theta_{0}\right)\delta_{n}+o(\delta_{n})\right]\\
=-\nabla^{2}g\left(\theta_{0}\right)\delta_{n}+O\left(\Norm{\delta_{n}}\right)\ensuremath{\delta_{n}}
\end{gather*}
and since$g$ is strictly convex and $h$ has a bounded Hessian on
$\Theta$ it holds that:
\[
\frac{\delta_{n}}{\alpha_{n}}\to v=-\left[\nabla^{2}g\left(\theta_{0}\right)\right]^{-1}\nabla h\left(\theta_{0}\right).
\]
Now, since$\frac{\Norm{\delta_{n}}}{\alpha_{n}}\to\Norm v$ and $\frac{\Norm{\delta_{n}}}{\alpha_{n}}\frac{\delta_{n}}{\Norm{\delta_{n}}}=\frac{\delta_{n}}{\alpha_{n}}\to v$
it holds that $\frac{\delta_{n}}{\Norm{\delta_{n}}}\to\frac{v}{\Norm v}.$
\end{proof}
\begin{lem}
\label{lem:strictly-convex}Let $l$ be a twice continuously differentiable
loss function and let$g$ be a twice continuously differentiable regularizer.
Assume $Z\left(\theta\right)=f_{\theta}\left(X\right)$, $\theta\in\R^{m}$.
For some $\theta_{0}\in\R^{m}$ assume that $Z\left(\theta_{0}\right)=0$,
$Z\left(\theta\right)$ is almost surely twice continuously differentiable
on an open neighborhood $\Omega$ of $\theta_{0}$ and that $\Expect{\nabla Z^{T}\left(\theta_{0}\right)\nabla Z\left(\theta_{0}\right)}\in\R^{m}$
is positive definite. Then, there is an open neighborhood $\Omega_{o}\subseteq\Omega$
of $\theta_{0}$ such that for small enough $\alpha\geq0$ it holds
that $\alpha{\cal L}_{l}\left(\theta\right)+\calG_{g}\left(\theta\right)$
is strictly convex on $\Omega_{0}$.
\end{lem}
\begin{proof}
It holds that:
\begin{multline*}
\nabla^{2}\text{\ensuremath{\left[\alpha{\cal L}_{l}+\calG_{g}\right]}}(\theta_{0})=\alpha\nabla^{2}{\cal L}\left(\theta_{0}\right)+\Expect{\sum_{i\in C}\pardiff g{z_{i}}\left(0\right)\nabla^{2}Z_{i}}\\
+\Expect{\nabla^{T}Z\left(\theta_{0}\right)\nabla^{2}g\left(0\right)\nabla Z\left(\theta_{0}\right)}\\
=\alpha\nabla^{2}{\cal L}\left(\theta_{0}\right)+\Expect{\nabla^{T}Z\left(\theta_{0}\right)\nabla^{2}g\left(0\right)\nabla Z\left(\theta_{0}\right)}
\end{multline*}
Now, since $\Expect{\nabla Z^{T}\left(\theta_{0}\right)\nabla Z\left(\theta_{0}\right)}\succ0$
and $\nabla^{2}g\left(0\right)\succ0$ it holds that: 
\[
\Expect{\nabla^{T}Z\left(\theta_{0}\right)\nabla^{2}g\left(0\right)\nabla Z\left(\theta_{0}\right)}\succ0,
\]
and thus there exists $\alpha_{0}>0$ such that:
\[
\nabla^{2}\text{\ensuremath{\left[\alpha_{0}{\cal L}_{l}+\calG_{g}\right]}}(\theta_{0}),\nabla^{2}\text{\ensuremath{\left[-\alpha_{0}{\cal L}_{l}+\calG_{g}\right]}}(\theta_{0})\succ0
\]
Thus, from the fact that $\nabla^{2}{\cal L}_{l}$ and $\nabla^{2}\calG_{g}$
are continues on $\Omega$ there is $\Omega_{0}\subseteq\Omega$,
an open neighborhood of $\theta_{0}$ such that for all $\theta\in\Omega_{0}$
it holds that:
\[
\nabla^{2}\left[\alpha_{0}\text{\ensuremath{{\cal L}_{l}}+\ensuremath{\calG_{g}}}\right]\left(\theta\right),\nabla^{2}\left[-\alpha_{0}\text{\ensuremath{{\cal L}_{l}}+\ensuremath{\calG_{g}}}\right]\left(\theta\right)\succ0
\]
and thus for all $-\alpha_{0}<\alpha<\alpha_{0}$ it holds that
\[
\nabla^{2}\left[\alpha\text{\ensuremath{{\cal L}_{l}}+\ensuremath{\calG_{g}}}\right]\left(\theta\right)\succ0.
\]
which is what we need.
\end{proof}

\subsection*{Proof of Lemma \ref{lem:allmost-robust}}
\begin{proof}
Let $Z^{\star}\in\argmin_{Z\in\calZ}\bar{L}_{l}\left(Z\right)+G_{g}\left(Z\right)$
and define $a=\left(1-\rho\frac{C}{C-1}\right)>0$. From Lemma \ref{lem:ghosh}
it holds that for any$Z\in\calZ$

\begin{multline*}
\left[\bar{L}\left(Z^{\star}\right)+G_{g}\left(Z^{\star}\right)\right]-\left[\bar{L}\left(Z\right)+G_{g}\left(Z\right)\right]\\
=\left[aL_{l}\left(Z^{\star}\right)+G_{g}\left(Z^{\star}\right)\right]-\left[aL_{l}\left(Z\right)+G_{g}\left(Z\right)\right].
\end{multline*}
From the optimality of $f^{\star}$ it holds that the left hand side
is smaller then 0 and after dividing by $a>0$ we have what we need.
\end{proof}

\subsection*{Proof of Theorem \ref{thm:assymptotic-local-robustness}.}
\begin{proof}
Let $0<\alpha_{n}\to0$ and $\Omega\ni\hat{\theta}_{n}\to\theta_{0}$
and that$\theta_{n}$ minimizes $\calL_{\alpha l+g}=\alpha\calL_{l}+{\cal G}_{g}$.
From the fact that $g$ is strictly convex and the assumption that
$\Expect{\nabla^{T}Z\left(\theta_{0}\right)\nabla Z\left(\theta_{0}\right)}$
is positive definite we have (by Lemma \ref{lem:strictly-convex}
taking \emph{$\alpha=0$}) that $\calG$ is strictly convex in $\theta_{0}$.
In addition, our assumption $\nabla Z\left(\theta_{0}\right)\left[\nabla^{2}\calG_{g}\left(\theta_{0}\right)\right]^{-1}\nabla\calL_{l}\left(\theta_{0}\right)\stackrel{\text{a.s.}}{\neq}0$
implies that $\nabla\calL_{l}\left(\theta_{0}\right)\ne0$, and thus
by Lemma\ref{lem:convergance_in_cos-1}. it holds for $\delta_{n}:=\theta_{n}-\theta_{0}$
that :
\[
\frac{\delta_{n}}{\Norm{\delta_{n}}}\to\bar{\delta}:=-\left[\nabla^{2}\calG_{g}\left(\theta_{0}\right)\right]^{-1}\nabla^{T}\calL_{l}\left(\theta_{0}\right)\ne0.
\]
By that and by our assumption that
\[
\nabla Z\left(\theta_{0}\right)\bar{\delta}=\nabla Z\left(\theta_{0}\right)\left[\nabla^{2}\calG_{g}\left(\theta_{0}\right)\right]^{-1}\nabla\calL_{l}\left(\theta_{0}\right)\stackrel{\text{a.s.}}{\neq}0,
\]
using a Taylor approximation of $Z_{n}$ we have that
\begin{align}
\frac{Z_{n}}{\Norm{Z_{n}}_{L^{2}}} & \stackrel{\text{a.s.}}{=}\frac{\nabla Z\left(\theta_{0}\right)\delta_{n}+O\left(\Norm{\delta_{n}}\right)\Norm{\delta_{n}}}{\norm{\nabla Z\left(\theta_{0}\right)\delta_{n}+O\left(\Norm{\delta_{n}}\right)\Norm{\delta_{n}}}_{L^{2}}}\label{eq:Z_n}\\
 & =\frac{\nabla Z\left(\theta_{0}\right)\frac{\delta_{n}}{\Norm{\delta_{n}}}+O\left(\Norm{\delta_{n}}\right)}{\Norm{\nabla Z\left(\theta_{0}\right)\frac{\delta_{n}}{\Norm{\delta_{n}}}+O\left(\Norm{\delta_{n}}\right)}_{L^{2}}}\nonumber \\
 & \stackrel{\text{a.s.}}{\to}\bar{Z}:=\frac{\nabla Z\left(Z_{0}\right)\bar{\delta}}{\norm{\nabla Z\left(Z_{0}\right)\bar{\delta}}_{L^{2}}}.\nonumber 
\end{align}

Let
\begin{gather}
\hat{l}_{\beta}=\beta l^{\text{lin}}+g^{\text{sq}},\label{eq:lhat-1}
\end{gather}
where
\[
l^{\text{lin}}\left(z,y\right)=\nabla_{z}l\left(0,y\right)z\text{ and }g^{\text{sq}}\left(z\right)=z^{T}\nabla^{2}g\left(z\right)z.
\]
By \ref{thm:unhinged-square} (setting $A=1/\beta z^{T}\nabla^{2}g\left(z\right)z$)
we have that $\hat{l}_{\beta}$ is robust to uniform label noise and
it is enough to show that there exists$0<\beta_{n}\to0$ and $\hat{\theta}_{n}\to\theta_{0}$
such that $\hat{\theta}_{n}$ is a local minimizer of $\calL_{\hat{l}_{\beta_{n}}}$
and 
\[
\frac{\hat{Z}_{n}}{\Norm{\hat{Z}_{n}}_{L^{2}}}\text{\ensuremath{\stackrel{\text{a.s.}}{\to}}}\bar{Z}
\]
where$\hat{Z}_{n}=Z\left(\hat{\theta}_{n}\right)$.

Since $Z$$\left(\theta\right)$ is almost surly twice continuously
differentiable on $\Omega$, from the smoothness assumption on $l$
and $g$, and from the assumptions that $Z\left(\theta_{0}\right)=0$
and $\Expect{\nabla^{T}Z\left(\theta_{0}\right)\nabla Z\left(\theta_{0}\right)}\succ0$
it holds by Lemma \ref{lem:strictly-convex} that for some $\beta_{0}>0$
if $\beta<\beta_{0}$ ${\cal L}_{\hat{l}_{\beta}}=\beta{\cal L}_{l^{\text{lin}}}+\calG_{g^{\text{sq}}}$
is strictly convex on $\Omega_{0}\subseteq$$\Omega$, a neighborhood
of $\theta_{0}$, and thus its minimizer converges to $\theta_{0}$.
We thus can choose some $\beta_{n}>0$, $\beta_{n}\to0$ and letting
$\hat{\theta}_{n}$ be the minimizer of ${\cal L}_{\hat{l}_{\beta}}$
on $\Omega_{0}$ for big enough $n$,when it is attained. By repeating
the argument in the beginning of the prof and substituting ${\cal L}_{l^{\text{lin}}}$,
$\calG_{g^{\text{sq}}}$, $\beta_{n}$ and $\hat{\theta}_{n}$ in
the places of $\calL_{l}$, $\calG_{g}$,$\alpha_{n}$ and $\theta_{n}$
we get what we need.
\end{proof}

\subsection*{Proof of Theorem \ref{thm:strong-robustness}.}
\begin{proof}
Let $\hat{l}_{\alpha}=\alpha l^{\text{lin}}+g^{\text{sq}}$ be as
defined in eq. \ref{eq:lhat-1}. Since $\Expect{\phi\left(X\right)\phi\left(X\right)^{T}}$
is positive definite $l_{\alpha}$ are $\hat{l}_{\alpha}$ uniquely
minimized at $\theta_{\alpha}$ and $\hat{\theta}_{\alpha}$ . From
the fact that the limit in eq. \ref{eq:Z-to-Zhat-alpha} holds what
we need is a direct consequences of the proof of Theorem\ref{thm:assymptotic-local-robustness}.
\end{proof}

\end{document}